\documentclass{article}

\usepackage[preprint, nonatbib]{neurips_2019}

\usepackage[utf8]{inputenc}
\usepackage[english]{babel}
\usepackage{bm}

\usepackage{amsmath}
\usepackage{amssymb}
\usepackage{amsthm}
\newtheorem{corollary}{Corollary}
\newtheorem{theorem}{Theorem}
\newtheorem{lemma}{Lemma}
\usepackage{mathtools}

\usepackage{algorithm}
\usepackage{algpseudocode}

\usepackage{graphicx}
\graphicspath{{./figures/}}
\usepackage{epstopdf}
\usepackage{wrapfig}

\usepackage{array}
\usepackage{multicol}
\usepackage{multirow}
\usepackage{booktabs}

\usepackage{hyperref}
\usepackage{url}

\newcommand{\diam}{\operatorname{diam}}
\newcommand{\Lip}{\operatorname{Lip}}
\newcommand{\Loss}{\mathcal{L}}
\newcommand{\E}{\mathbb{E}}
\newcommand{\N}{\mathbb{N}}
\newcommand{\R}{\mathbb{R}}
\newcommand{\bt}{\bm{\beta}}
\newcommand{\hM}{\widehat{M}}
\newcommand{\hx}{\mathbf{\hat{\xi}}}
\newcommand{\m}{\mathbf{m}}
\newcommand{\n}{\bm{\nu}}
\newcommand{\om}{\bm{\omega}}
\newcommand{\one}{\mathbf{1}}
\newcommand{\ph}{\varphi}
\newcommand{\s}{\mathbf{s}}
\newcommand{\sg}{\varsigma}

\newcommand{\w}{\mathbf{w}}
\newcommand{\x}{\mathbf{x}}
\newcommand{\xx}{\bm{\xi}}
\newcommand{\y}{\mathbf{y}}
\renewcommand{\b}{\mathbf{b}}
\renewcommand{\hm}{\mathbf{\hat{m}}}
\renewcommand{\l}{\ell}

\newdimen{\algindent}
\algnewcommand\commentspace[1]{\hspace{#1\algindent}}

\newcommand{\onedtnsrshp}[1]{#1}
\newcommand{\threedtnsrshp}[3]{#1 $\times$ #2 $\times$ #3}
\usepackage[flushleft]{threeparttable}
\usepackage{makecell}
\usepackage{pifont}

\title{Additive Noise Annealing and Approximation Properties of Quantized Neural Networks}
\author{%
  Matteo Spallanzani\\
  Università di Modena e Reggio Emilia\\
  {\small \texttt{matteo.spallanzani@unimore.it}}\\
  \And
  Lukas Cavigelli \\
  ETH Zürich\\
  {\small \texttt{cavigelli@iis.ee.ethz.ch}}\\
  \And
  Gian Paolo Leonardi\\
  Università di Trento\\
  {\small \texttt{gianpaolo.leonardi@unitn.it}}\\
  \And
  Marko Bertogna\\
  Università di Modena e Reggio Emilia\\
  {\small \texttt{marko.bertogna@unimore.it}}\\
  \And
  Luca Benini\\
  ETH Zürich \& Università di Bologna\\
  {\small \texttt{benini@iis.ee.ethz.ch}}
}

\begin{document}

\maketitle

\begin{abstract}
\noindent
We present a theoretical and experimental investigation of the quantization problem for artificial neural networks.
We provide a mathematical definition of quantized neural networks and analyze their approximation capabilities, showing in particular that any Lipschitz-continuous map defined on a hypercube can be uniformly approximated by a quantized neural network.
We then focus on the regularization effect of additive noise on the arguments of multi-step functions inherent to the quantization of continuous variables.
In particular, when the expectation operator is applied to a non-differentiable multi-step random function, and if the underlying probability density is differentiable (in either classical or weak sense), then a differentiable function is retrieved, with explicit bounds on its Lipschitz constant.
Based on these results, we propose a novel gradient-based training algorithm for quantized neural networks that generalizes the straight-through estimator, acting on noise applied to the network's parameters.
We evaluate our algorithm on the CIFAR-10 and ImageNet image classification benchmarks, showing state-of-the-art performance on AlexNet and MobileNetV2 for ternary networks.
The code to replicate our experiments and apply the algorithm is available online \footnote{\href{https://github.com/spallanzanimatteo/QuantLab}{{https://github.com/spallanzanimatteo/QuantLab}}}.
\end{abstract}

\section{Introduction}
\label{sec:introduction}
Deep learning (DL) \cite{Hinton2007, Bengio2009} is the backbone of present-day artificial intelligence (AI).
As general function approximators, deep neural networks (DNNs) are used as critical components whenever analytical algorithms to solve a given task are not available.
For example, specific DNN architectures have been developed to perform object detection \cite{Redmon2018, Pan2018}, autonomous navigation \cite{Gandhi2017, Loquercio2018} or decision making in both complete and incomplete information games \cite{Silver2017, Vinyals2019}.

Designed to optimize statistical fit metrics, these models often require billions of parameters and multiply-accumulate (MAC) operations to perform inference on a single data sample.
These properties translate into often prohibitive energy and latency requirements for resource-constrained devices.
Many real-world applications in the context of embedded intelligence (internet of things, smart cities) have additional constraints that need to be satisfied, yielding more complex optimization metrics.
Energy-aware applications must satisfy limited peak-power constraints or average energy consumption constraints \cite{Mayer2019}.
Real-time applications must satisfy specific inference rate constraints \cite{Quinones2018}.
Many applications are both energy-constrained and real-time constrained \cite{Palossi2019}.
Constrained AI has thus become a very active research field.

In particular, constrained DL is concerned with the development of models which do not exceed a given computational cost and have a limited memory footprint, a property that is a function of both the number of parameters and the precision of their hardware representations.
The number of parameters and the computational complexity have been extensively studied in the context of network topologies, in particular for convolutional neural networks (CNNs).
CNNs have received considerable attention from the advent of AlexNet \cite{Krizhevsky2012}.
In particular, the simplification of filters used \cite{Simonyan2015} and the introduction of residual connections modelled by smaller networks \cite{He2016, Lin2014} have opened the way to accurate modern architectures for computer vision applications such as MobileNetV2 \cite{Sandler2018}, a CNN designed explicitly for mobile devices.
More recently, advanced neural architecture search algorithms have been used to develop extremely hardware-friendly topologies \cite{Wu2018}.
Hardware-related compression has recently attracted much interest \cite{Agustsson2017, Cavigelli2018}, evolving into the field of quantized neural networks (QNNs).
QNNs use low-bitwidth operands to reduce the models' memory footprints and simplify the hardware arithmetic needed to perform inference.

In this paper, we present the following contributions:
\begin{enumerate}
    \item We formulate a concise but general mathematical description of QNNs and prove the first (to the best of the authors' knowledge) approximation theorem for these models. QNNs can approximate the same set of functions that can be approximated by continuous-valued DNNs (formally speaking, QNNs are dense in the space of continuous functions defined on a hypercube).
    \item We theoretically formalize the existing intuitions about the probabilistic nature of the straight-through estimator (STE) \cite{Bengio2013} and quantify the regularization effect induced by noise on non-differentiable functions. The associated probability density must satisfy very mild regularity assumptions (mathematically speaking, it only needs to be a BV function) and can produce STE as a very specific case.
    \item We design a new gradient-based training algorithm based on this regularization effect, which we name \emph{Additive Noise Annealing} (ANA).
    \item We report our experiments with ternary CNNs, where the values of weights and representations take values in $Q = \{-1, 0, +1\}$. We tested ANA on the established image classification benchmarks CIFAR-10 (validation accuracy of $90.74\%$ using a VGG-like architecture) and ImageNet (Top-1 validation accuracies of $45.80\%$ on AlexNet and $64.79\%$ on MobileNetV2) and compare the results to the state-of-the-art.
\end{enumerate}

\section{Related work}
\label{sec:positioning}
The numeric variables that represent the operands of a DNN are partitioned in parameters and representations (i.e., the outputs of activation functions).
Parameters are further partitioned in weights and biases.
From a mathematical viewpoint, traditional full-precision neural networks parameters take values in continuous spaces.
The commonly used activation functions (sigmoid, hyperbolic tangent, ReLU) have continuous real-valued codomains, which by definition yield continuous representations.
We call quantization set any finite non-empty set $Q$ of real numbers, called quantization levels.
A weights-$Q$-quantized DNN is a DNN whose weights take values in $Q$.
An activations-$Q$-quantized DNN is a DNN whose activation functions have $Q$ as their codomain.
With these simple definitions, the problem of DNN quantization can be defined as the investigation of neural networks models that are weights-$Q$-quantized, activations-$Q$-quantized, or both.

The two most popular algorithms to train weights-$Q$-quantized DNNs are \emph{binary connect} (BC) \cite{Courbariaux2015} and \emph{incremental network quantization} (INQ) \cite{Zhou2017}.
BC is given a network architecture initialized with \textit{shadow} full-precision weights.
Before inference is performed, a quantization function $\zeta$ (the generalized Heaviside with codomain $Q = \{-1, +1\}$, also known as the sign function) is applied to these \textit{shadow} weights and returns binary values.
Data is then propagated forward using these values as weights, the loss is evaluated, and gradients are computed.
The updates directed to the binary weights are then applied to the \textit{shadow} weights instead, passing through the sign functions unaltered.
More formally, STE replaces the exact distributional derivative $D\zeta = 2\delta_{0}$ of the sign function with the derivative of the hard hyperbolic tangent:
\begin{equation*}
  \frac{d}{d x} \zeta(x) \approx
  \begin{cases}
    1, &\text{if} \,\, x \in [-1, 1] \\
    0, &\text{if} \,\, x \notin[-1, 1]
  \end{cases} \,.
\end{equation*}
INQ is given a network architecture initialized with full-precision weights.
The algorithm defines a quantization set $Q$, partitions the weights into a fixed number of subsets and defines a corresponding number of quantization time steps.
INQ then starts training the full-precision model.
When a quantization time step is reached, the corresponding weights subset is \textit{frozen} (i.e., its elements are projected onto the nearest quantization levels and never updated again).

As an evolution of BC, activations-$Q$-quantized DNNs were first investigated by Hubara \emph{et al.} \cite{Hubara2018}.
They used the sign function also as the activation function in all the models they tested, and STE was effectively applied both to weights and activation functions.
Building on STE, many algorithms have been developed.
XNOR-Net \cite{Rastegari2016} tried to preserve inner products of full-precision filters and representation vectors by projecting them onto optimal binary vectors.
Elaborating on this idea, ABC-Net \cite{Lin2017} decomposed filters onto a suitable binary basis and tried to learn multi-step activation functions that could be expressed as weighted sums of Heaviside functions.
Taking the concept even further, \emph{compact neural networks} \cite{Zhuang2018} decomposed entire residual branches into multiple branches modelled by binary weights and multi-step activation functions.
Interestingly, GXNOR-Net \cite{Deng2018} presented a solution to update weights that does not use continuous \textit{shadow} weights, but uses the information carried by gradients to perform discrete state transitions instead.

The transition from continuous to finite spaces wipes out the differentiability property on which the backpropagation algorithm \cite{Rumelhart1986} is based.
In this context, the learning problem is converted from a numerical optimization problem into a discrete optimization one, where the state space is finite but very large.
Gradient-free optimization algorithms, such as integer programming or Monte Carlo methods, quickly become unfeasible due to the combinatorial explosion of configurations.
Although not formally correct in the context of discrete optimization, gradient-based optimization has been effectively brought back into the picture by STE.
Nonetheless, the reason for its success is still mathematically unclear.

\section{A mathematical formulation of QNNs}
\label{sec:model}
We start by introducing some general notation and terminology for the composition of parametric maps.
Let $L \geq 2$ be an integer and let $X^{0}, \dots, X^{L}$ and $M^{1}, \dots, M^{L}$ be nonempty sets.
The sets $X^{\l}$ are the \textbf{representation spaces}, while the sets $M^{\l}$ are the \textbf{parameter spaces}.
For each $\l = 1, \dots, L$ we consider a map $\psi^{\l}$ defined on pairs $(\m^{\l}, \x^{\l-1})\in M^{\l} \times X^{\l-1}$ with values in $X^{\l}$.
For a fixed $\m^{\l} \in M^{\l}$ we set $\psi_{\m^{\l}}(\x^{\l-1}) = \psi^{\l}(\m^{\l}, \x^{\l-1})$, so that $\psi_{\m^{\l}}$ is a map from $X^{\l-1}$ to $X^{\l}$, parametrized by $\m^{\l}$.
Hence, every $L$-tuple $(\m^{1}, \dots, \m^{L}) \in M^{1} \times \dots \times M^{L}$ defines a corresponding $L$-tuple $(\psi_{\m^{1}}, \dots, \psi_{\m^{L}})$ of maps, that can be composed as follows:
\begin{equation}\label{eq:composed_map}
  X^{0} 
  \xrightarrow {\psi_{\m^{1}}} X^{1}
  \xrightarrow {\psi_{\m^{2}}} X^{2}
  \dots
  X^{L-1} 
  \xrightarrow {\psi_{\m^{L}}} X^{L} \,.
\end{equation}
Set $\hm^{\l} = (\m^{1}, \dots, \m^{\l})$ for $\l \in \{1, \dots, L\}$. Define by induction $\Psi_{\hm^{1}} = \psi_{\m^{1}}$ and $\Psi_{\hm^{\l}} = \psi_{\m^{\l}} \circ \Psi_{\hm^{\l-1}}$ for $\l \in \{2, \dots, L\}$, where $\circ$ is the composition operator.
In particular, the map in \eqref{eq:composed_map} corresponds to $\Psi_{\hm} \,:\,X^{0} \to X^{L}$ where we have set $\hm = \hm^{L}$ for brevity.
We remark that we are not enforcing other properties on these maps other than their parametric nature and their composability structure.

When $X^{\l} = \R^{n_{\l}}$ for $\l = 0, \dots, L$, we will consider a class of more specific maps $\ph_{\m^{\l}} \,:\, X^{\l-1} \to X^{\l}$ of the form
\begin{equation}\label{eq:fnn_layer}
  \ph_{\m^{\l}}(\x^{\l-1}) = \left( \sigma^{\l} \Big( \sum_{j=1}^{n_{\l-1}} w^{\l}_{1j} \, x^{\l-1}_{j} + b^{\l}_{1} \Big), \dots, \sigma^{\l} \Big( \sum_{j=1}^{n_{\l-1}} w^{\l}_{n_{\l} j} \, x^{\l-1}_{j} + b^{\l}_{n_{\l}} \Big) \right) \,,
\end{equation}
where $\w^{\l} \in \R^{n_{\l} \times n_{\l-1}}$ is a matrix of weights and $\b^{\l} \in \R^{n_{\l}}$ is a vector of biases.
One can define $\Phi_{\hm^{\l}}$ and $\Phi_{\hm}$ similarly as in \eqref{eq:composed_map}.
The function $\sigma^{\l} \,:\,\R \to \R$ appearing in \eqref{eq:fnn_layer} is called \textbf{activation function}; it is assumed to be bounded, non-constant and usually non-decreasing \cite{Cybenko1989, Hornik1991}, in all the layers $\ph_{\m^{\l}}, \l = 1, \dots L-1$ except for the last, where it is usually replaced by the identity function (i.e., $\ph_{\m^{L}}$ is an affine function).
If we identify $\m^{\l}$ with the pair $(\w^{\l}, \b^{\l})$, we obtain the more compact expression
\begin{equation}\label{eq:phi_classic}
  \ph_{\m^{\l}}(\x^{\l-1}) = \sigma^{\l} \left( \w^{\l} \cdot \x^{\l-1} + \b^{\l} \right)
\end{equation}
where the dot product indicates the usual matrix-vector product and $\sigma^{\l}(s_{1}, \dots, s_{n}) = (\sigma^{\l}(s_{1}), \dots, \sigma^{\l}(s_{n}))$.
In this case, the map $\Phi_{\hm}$ is called a \textbf{feedforward neural network} (FNN) with \textbf{input space} $X^{0}$, \textbf{output space} $X^{L}$, \textbf{hidden spaces} $X^{1}, \dots, X^{L-1}$, and \textbf{parameter space} $\hM = M^{1} \times \dots \times M^{L}$. 

Let $Q = \{q_{0} < q_{1} < \dots < q_{K}\} \subset \R$ be a finite set of real numbers.
We call $Q$ the \textbf{quantization set}, and its elements quantization levels.
We say that a matrix $\w \in \R^{m \times n}$ is $Q$-quantized if $w_{ij} \in Q$ for all $i, j$.
Similarly, we say that a function $\sigma \,:\, \R \to \R$ is $Q$-quantized if its codomain $\sigma(\R)$ is $Q$.
We will more generally call \textbf{quantization function} any function that is $Q$-quantized for some $Q$.
We say that the map $\ph_{\m^{\l}}$ defined in \eqref{eq:phi_classic} is \textbf{weight-$Q$-quantized} if $\w^{\l}$ is $Q$-quantized.
We say that $\ph_{\m^{\l}}$ is \textbf{activation-$Q$-quantized} if $\sigma^{\l}$ is $Q$-quantized.
Then, the layer $\ph_{\m^{\l}}$ is said \textbf{$Q$-quantized} if $\w^{\l}$ and $\sigma^{\l}$  are $Q$-quantized.
Notice that $\b^{\l}$ is not required to be $Q$-quantized.
We say that $\Phi_{\hm}$ is a \textbf{quantized neural network} if $L \ge 2$ and the layers $\ph_{\m^{\l}}$ are $Q$-quantized for every $\l = 1, \dots, L-1$.
Notice that the last layer $\ph_{\m^{L}}$ is not required to be $Q$-quantized.
In particular, when $Q = \{-1, +1\}$ we say that $\Phi_{\hm}$ is a \textbf{binary neural network} (BNN).
Similarly when $Q = \{-1, 0, +1\}$ we say that $\Phi_{\hm}$ is a \textbf{ternary neural network} (TNN).
We now introduce a useful class of quantization functions.
A \textbf{generalized Heaviside function} is the map:
\begin{equation*}
  H_{\theta}^{\{q_{0}, q_{1}\}}(x) = 
  \begin{cases}
    q_{0}, &\text{if} \,\, x < \theta \\
    q_{1}, &\text{if} \,\, x \geq \theta
  \end{cases} \,.
\end{equation*}
The standard Heaviside function is recovered when $q_{0} = 0, q_{1} = 1$ and $\theta = 0$; for simplicity, we will refer to it with the symbol $H(x)$.
We observe that a generalized Heaviside can be expressed in terms of the canonical Heaviside by $H_{\theta}^{\{q_{0}, q_{1}\}}(x) = q_{0} + \delta q_{1} H(x - \theta)$, where $\delta q_{1} = q_{1} - q_{0}$.
Let $\Theta = \{\theta_{1} < \theta_{2} < \dots < \theta_{K}\} \subset \R$ be a finite ordered set of real thresholds.
Let $Q = \{q_{0} < q_{1} < \dots < q_{K}\} \subset \R$ be a quantization set.
The \textbf{$K$-step function} with thresholds $\Theta$ and quantization set $Q$ is the non-decreasing map $\sigma \,:\, \R \to Q$ defined as
\begin{equation}\label{eq:multistep_linear}
  \sigma(x) = q_{0} + \sum_{k=1}^{K} \delta q_{k} H(x - \theta_{k}) \,,
\end{equation}
where $\delta q_{k} = q_{k} - q_{k-1}$ are the jumps between consecutive quantization levels.
The generic term \textbf{multi-step function} will denote a function of the form \eqref{eq:multistep_linear}.
We can use multi-step functions to describe QNNs.
Suppose $\w^{\l} \in \R^{n_{\l} \times n_{\l-1}}$ is a real-valued matrix.
If we consider a weight quantization function $\zeta^{\l}$ and an activation function $\sigma^{\l}$ of the form \eqref{eq:multistep_linear}, we can define the map
\begin{equation}\label{eq:phi_quantized}
  \ph_{\m^{\l}}(\x^{\l-1}) = \sigma^{\l} \left( \zeta^{\l}(\w^{\l}) \cdot \x^{\l-1} + \b^{\l} \right) \,.
\end{equation}
This layer is clearly $Q$-quantized since both $\zeta^{\l}(\w^{\l})$ and $\sigma^{\l}$ are $Q$-quantized.
Although the idea of a weight quantization function seems unusual with respect to the common definition of parameters in neural networks, yet this model is consistent with many of the QNNs training algorithms based on STE \cite{Hubara2018, Rastegari2016, Lin2017, Zhuang2018}.
Notice also that the traditional layer \eqref{eq:phi_classic} can be recovered from \eqref{eq:phi_quantized} if the multi-step function $\zeta^{\l}$ is replaced by the identity function.
Replacing layers of the form \eqref{eq:phi_classic} with layers of the form \eqref{eq:phi_quantized} in the map \eqref{eq:composed_map} yields a QNN.

Traditional approximation results for ANNs  \cite{Cybenko1989, Hornik1991} assume continuous-valued parameters.
A first natural question to investigate about QNNs is which function classes they can approximate.
\begin{theorem}[Uniform approximation by QNNs]\label{th:approximation}
Let $X^{0} = [0, S]^{n_{0}} \subset \R^{n_{0}}$ and denote by $\Lip_{\lambda}(X^{0})$ the class of bounded functions $f \,:\, X^{0} \to [-\lambda, \lambda]$ with Lipschitz constant $\leq \lambda$.
For every $\epsilon > 0$ and $f \in \Lip_{\lambda}(X^{0})$, there exists a network
\begin{equation*}
  \Phi_{\hm} = \ph_{\m^{3}} \circ \ph_{\m^{2}} \circ \ph_{\m^{1}} \,:\, X^{0} \to [-\lambda, \lambda] \,,
\end{equation*}
composed by two layers of the form \eqref{eq:phi_classic} (with quantized parameters $\w^{1} \in \{-1, 0, +1\}^{2n_{0} \times n_{0}}$ and $\w^{2} \in \{0, 1\}^{N \times n_{0}}, N = N(\epsilon)$, and activation function $\sigma(x) = H(x)$) followed by an affine map $\ph_{\m^{3}}$ (parametrized by real parameters), such that $\sup_{\x \in X^{0}} |\Phi_{\hm}(\x) - f(\x)| \le \epsilon$.
\end{theorem}
Remarkably, since $\cup_{\lambda}\Lip_{\lambda}(X^{0})$ is dense in $C^{0}(X^{0})$, QNNs show approximation capabilities equivalent to those of full-precision networks.
The proof is given in Appendix~\ref{app:proofs}.

The \textbf{supervised learning} problem \cite{Vapnik1998} to be solved is the minimization of the \textbf{loss functional}
\begin{equation}\label{eq:loss}
  \Loss_{g, \gamma}(\Phi_{\hm}) = \int_{X^{0} \times X^{L}} d(\Phi_{\hm}(\x^{0}), g(\x^{0}))\, d\gamma(\x^{0}) \,,
\end{equation}
where $(\x^{0}, \y = g(\x^{0}))$ is a sampled observation that associates an input instance $\x^{0}$ with its label $\y$ (obtained through an unknown \textit{oracle} function $g \,:\, X^{0} \to X^{L}$), $d$ is a differentiable non-negative function called the \textbf{loss function}, and $\gamma$ is a probability measure defined on $X^{0} \times X^{L}$.
In practical applications, the measure $\gamma$ is approximated by the empirical measure $\tilde{\gamma}(\x) = \sum_{t=1}^{T} \delta_{\x^{0}_{t}}(\x) / T$ obtained from the finite dataset $\{ \x^{0}_{t} \}_{t = 1, \dots, T}$.
By \textbf{training} of a FNN we mean any algorithm that aims at minimizing $\Loss_{g, \gamma}(\Phi_{\hm})$ as a function of $\hm \in \hM$.
For example, when $\Phi_{\hm}$ is differentiable with respect to its parameters $\hm$, the gradient $\nabla_{\hm} \Loss_{g, \gamma}(\Phi_{\hm})$ can be computed applying the chain rule (due to the compositional structure of $\Phi_{\hm}$) as in the backpropagation algorithm \cite{Rumelhart1986}, and $\hm$ can be updated via gradient descent optimization.
However, the chain rule can no longer be applied to a QNN, where non-differentiable building blocks of the form (5) are used.
\emph{A multi-step function \eqref{eq:multistep_linear} is not differentiable at the thresholds, since its distributional derivative is a weighted sum of Dirac's deltas centered on the thresholds; interestingly, when noise satisfying certain regularity properties is added to the argument of \eqref{eq:multistep_linear} and the expectation operator is applied, the resulting function turns out to be Lipschitz or even differentiable in classical sense}.
The standard mathematical notations and definitions that are needed from now on (i.e., that of $L^p$ functions, of distributional derivative, of Sobolev and of BV functions) are recalled in Appendix \ref{app:proofs}.
\begin{theorem}[Noise regularization effect on multi-step functions]\label{th:regularization}
Let $\sigma \,:\, \R \to Q$ be a multi-step function.
Let $\nu$ be a zero-mean noise with density $\mu(-\nu)$.
Define the random function $\sigma_{\nu}(x) = \sigma(x + \nu)$ for $x\in \R$.
Then:
\begin{itemize}
    \item[(i)] $\E_{\mu}[\sigma_{\nu}] = \sigma \ast \mu$;
    \item[(ii)] $\mu \in W^{1,1}(\R)$ implies that $\E_{\mu}[\sigma_{\nu}]$ is differentiable, its derivative is bounded, continuous, and satisfies $\frac{d}{dx}\E_{\mu}[\sigma_{\nu}] = \E_{\mu}[D\sigma_{\nu}] = \sigma \ast D\mu$;
    \item[(iii)] $\mu \in BV(\R)$ implies that $\E_{\mu}[D\sigma_{\nu}] \in Lip_{\lambda}(\R)$, where $\lambda = |D\mu|(\R)$ is the total variation of $D\mu$.
\end{itemize}
\end{theorem}
In other words, the expectation acts as a convolution on the non-differentiable function $\sigma$ (with the noise playing the role of the kernel).
Therefore, the regularity of the noise $\mu$ (that we identify with the probability density, up to a slight abuse of notation) is transferred to the expectation of the random function $\sigma_\nu$.
The proof is given in Appendix~\ref{app:proofs}.
Notice that for symmetric distributions, such as the uniform or the Gaussian, $\mu(-\nu) = \mu(\nu)$.

Let us define random variables $\n_{\w^{\l}}$ and $\n_{\b^{\l}}$ distributed according to zero-mean probability measures $\mu_{\w^{\l}}$ and $\mu_{\b^{\l}}$ respectively.
We use these random variables as additive noises to turn the parameters $\w^{\l}$ and $\b^{\l}$ of a layer $\ph_{\m^{\l}}$ into random variables $\om^{\l} = \w^{\l} + \n_{\w^{\l}}, \bt^{\l} = \b^{\l} + \n_{\b^{\l}}$ where $\w^{\l}$ and $\b^{\l}$ represent the means.
These parameters are distributed according to the translated measures $\mu_{\om^{\l}}(\om) = \mu_{\w^{\l}}(\om - \w^{\l}), \mu_{\bt^{\l}}(\bt) = \mu_{\b^{\l}}(\bt - \b^{\l})$.
We define the symbol $\xx^{\l}$ to represent the tuple $(\om^{\l}, \bt^{\l})$.
We are thus describing the parameters of $\ph_{\m^{\l}}$ as events $\xx^{\l}$ belonging to a probability space $\Xi^{\l}$, on which a probability measure $\mu^{\l}$ (for example, $\mu^{\l} = \mu_{\om^{\l}} \times \mu_{\bt^{\l}}$) is defined.
We can thus redefine \eqref{eq:phi_quantized} as a random function:
\begin{equation}\label{eq:phi_noisy}
  \ph_{\xx^{\l}}(\x^{\l-1}) = \sigma^{\l} \left( \zeta^{\l}(\om^{\l}) \cdot \x^{\l-1} + \bt^{\l} \right) \,.
\end{equation}
For a given $L$-layer network composed by maps of this form, we define the parameter tuple $\hx = (\xx^{1}, \dots, \xx^{L})$ and the parameter space $\hat{\Xi} = \Xi^{1} \times \dots \times \Xi^{L}$.
On this space, we naturally define the product measure $\hat{\mu} = \mu^{1} \times \dots \times \mu^{L}$.
The $L$-layer \textbf{stochastic feedforward neural network} (SFNN) is defined by the symbol $\Phi_{\hx}$.
We observe that the deterministic case \eqref{eq:phi_quantized} is retrieved when the measure $\hat{\mu}$ is the product of Dirac's deltas concentrated at the parameters means:
\begin{equation}\label{eq:delta_product}
  \delta^{\hm} = \delta_{\m^{1}} \times \dots \times \delta_{\m^{L}} \,.
\end{equation}
In this sense, SFNNs generalize classic FNNs, and in particular \eqref{eq:phi_noisy} generalizes \eqref{eq:phi_quantized}.

\section{The \emph{Additive Noise Annealing} algorithm}
\label{sec:algorithm}
\begin{wrapfigure}[32]{R}{.5\textwidth}
\begin{minipage}{.5\textwidth}
  \begin{algorithm}[H]
    \caption{Additive Noise Annealing}
    \hspace*{\algorithmicindent} \textbf{Input} $\Phi_{\hm_{0}}, \lambda_{0}, T, \{(\x^{0}_{t}, \y_{t})\}_{t = 1, \dots, T}$ \\
    \hspace*{\algorithmicindent} \textbf{Output} $\Phi_{\hm_{T}}$
    \begin{algorithmic}[1]
      \For{$t \gets 1, T$} \label{algo:ana_training_s}
        \For{$\l \gets 1, L$} \label{algo:ana_get_noise_s} \Comment{compute noise}
          \State $\mu_{\om^{\l}}^{f}, \mu_{\om^{\l}}^{b} \gets get\_noise(t, \w^{\l})$
          \State $\mu_{\bt^{\l}}^{f}, \mu_{\bt^{\l}}^{b} \gets get\_noise(t, \b^{\l})$
        \EndFor \label{algo:ana_get_noise_e}
        \For{$\l \gets 1, L$} \label{algo:ana_inference_s} \Comment{infer}
          \State $\tilde{\w}^{\l} \gets \E_{\mu_{\om^{\l}}^{f}}[\zeta^{\l}(\om^{\l})]$
          \State $\x^{\l}_{t} \gets \E_{\mu_{\bt^{\l}}^{f}}[\sigma^{\l}(\tilde{\w}^{\l} \cdot \x^{\l-1}_{t} + \bt^{\l})]$
        \EndFor \label{algo:ana_inference_e}
        \State $g_{\x^{L}} \gets \nabla_{\x^{L}}d_{L}(\x^{L}_{t}, \y_{t})$ \label{algo:ana_backprop_s} \Comment{backpropagate}
        \For{$\l \gets L, 1$}
          \State $g_{\s^{\l}} \gets g_{\x^{\l}} \cdot \nabla_{\s^{\l}}\E_{\mu_{\bt^{\l}}^{b}}[\sigma^{\l}(\s^{\l} + \bt^{\l})]$
          \State $g_{\b^{\l}} \gets g_{\s^{\l}}$
          \State $g_{\tilde{\w}^{\l}} \gets g_{\s^{\l}} \cdot \x^{\l-1}_{t}$
          \State $g_{\w^{\l}} \gets g_{\tilde{\w}^{\l}} \cdot \nabla_{\w^{\l}}\E_{\mu_{\om^{\l}}^{b}}[\zeta^{\l}(\om^{\l})]$
          \State $g_{\x^{\l-1}} \gets g_{\s^{\l}} \cdot \tilde{\w}^{\l}$
        \EndFor \label{algo:ana_backprop_e}
        \State $\hm_{t} \gets optim(\lambda_{t-1}, \hm_{t-1}, g_{\hm})$
        \State $\lambda_{t} \gets lr\_sched(t, \lambda_{t-1})$
      \EndFor \label{algo:ana_training_e}
      \State \Return $\Phi_{\hm_{T}}$
    \end{algorithmic}
    \label{algo:ana}
  \end{algorithm}
\end{minipage}
\end{wrapfigure}
A natural way to perform inference using a stochastic network $\Phi_{\hx}$ is applying the expectation operator to it, i.e. defining:
\begin{equation}\label{eq:variational}
  \E_{\hat{\mu}}[\Phi_{\hx}(\x^{0})] = \int_{\hat{\Xi}} \Phi_{\hx}(\x^{0}) d\hat{\mu}(\hx) \,.
\end{equation}
When the measure $\hat{\mu}$ takes the form \eqref{eq:delta_product}, this expectation corresponds to a deterministic evaluation of a QNN.
Thus, an algorithm designed to train QNNs should implement a search for optimal values of the means $\hm$.
However, no closed-form computable expression for \eqref{eq:variational} is available, so exact inference is in general not possible.
Further, developing a tool equivalent to the chain rule to compute
\begin{equation}\label{eq:variational_gradient}
  \nabla_{\hm} \E_{\hat{\mu}}[\Phi_{\hx}(\x^{0})]
\end{equation}
is unfeasible.
We thus needed to replace exact evaluations of \eqref{eq:variational} and \eqref{eq:variational_gradient} with approximations.
The expectation and composition operators do not, in general, commute.
But assuming that each layer map $\ph_{\xx^{\l}} \,:\, X^{\l-1} \to X^{\l}$ is continuous uniformly with respect to $\xx^{\l}$, the next theorem states that the composition of expectations $\E_{\mu^L}[\ph_{\xx^{L}}] \circ \cdots \circ \E_{\mu^1}[\ph_{\xx^{1}}]$ pointwise converges to the (deterministic) feedforward neural network $\Phi_{\hm}$ as soon as $\hat{\mu} = \mu^1 \times \cdots \times \mu^L$ converges to \eqref{eq:delta_product}.
\begin{theorem}[Continuity of composed expectations]\label{th:continuity}
Let $\Xi{^\l}$ and $X^{\l-1}$ be compact subsets of some Euclidean spaces.
Assume that, for all $\l = 1, \dots, L$, the map $\ph_{\xx^{\l}}(\x^{\l-1}) = \ph^{\l}(\xx^{\l}, \x^{\l-1})$ is continuous in both variables $\xx^{\l}$ and $\x^{\l-1}$.
Let $\{{\mu}^{\l}_{t}\}_{t \in \N}$ be a sequence of probability measures on $\Xi^{\l}$ converging to the Dirac's delta $\delta_{\m^{\l}}$ for suitable $\m^{\l} \in \Xi^{\l}$ and for $\l = 1, \dots, L$.
Then $\lim_{t \to \infty} \E_{\mu^{L}_{t}}[\ph_{\xx^{L}}] \circ \cdots \circ \E_{\mu^{1}_{t}}[\ph_{\xx^{1}}](\x) = \Phi_{\hm}(\x), \,\forall\, \x \in X^{0}$.
\end{theorem}
Even though the $\sigma^{\l}$ used to model QNNs are not continuous, Theorem~\ref{th:continuity} motivates ANA.
At every training iteration (Algorithm~\ref{algo:ana}, lines \ref{algo:ana_training_s}-\ref{algo:ana_training_e}), ANA sets probability measures $\mu_{\om^{\l}} = \mu_{\om^{\l}}^{f} = \mu_{\om^{\l}}^{b}, \mu_{\bt^{\l}} = \mu_{\bt^{\l}}^{f} = \mu_{\bt^{\l}}^{b}$ for each parameter $\om^{\l}$ and $\bt^{\l}$ (lines \ref{algo:ana_get_noise_s}-\ref{algo:ana_get_noise_e}).
ANA then computes the stack of linear and nonlinear functions that compose the network, taking expectations as soon as the respective functions are applied (lines \ref{algo:ana_inference_s}-\ref{algo:ana_inference_e}).
During backpropagation (lines \ref{algo:ana_backprop_s}-\ref{algo:ana_backprop_e}), local gradients are computed using the rule provided by Theorem~\ref{th:regularization}, unstacking the composition of functions.
We remark that the measures $\mu_{\om^{\l}}, \mu_{\bt^{\l}}$ depend on time, and should be annealed to Dirac's deltas as $t \to T$.
We discuss the implemented annealing strategies in Section~\ref{sec:experiments}.

Using Theorem~\ref{th:regularization}, we can interpret STE \cite{Bengio2013} as a particular instance that applies noise to the argument of the sign function according to two different distributions $\mu^{f}, \mu^{b}$ during the forward and backward passes.
\begin{corollary}\label{th:ste}
Let $\sigma(x) = H_{0}^{\{-1, +1\}}(x)$ denote the univariate sign function.
We define forward and backward probability distributions $\mu^{f} = \delta_{0}$ and $\mu^{b} = \mathcal{U}[-1, +1]$, and denote with $\nu^{f}, \nu^{b}$ the additive noises distributed accordingly.
By Theorem~\ref{th:regularization} we have $\E_{\mu^{f}}[\sigma(x+\nu^{f})] = \sigma(x)$ and
\begin{equation*}
  \frac{d}{d x} \E_{\mu^{b}}[\sigma(x+\nu^{b})] =
  \begin{cases}
    1, &\text{if} \,\, x \in [-1, 1] \\
    0, &\text{if} \,\, x \notin[-1, 1]
  \end{cases} \,.
\end{equation*}
\end{corollary}
This observation led us to devise a generalized version of Algorithm~\ref{algo:ana} that allows defining different measures $\mu_{\om^{\l}}^{f} \neq \mu_{\om^{\l}}^{b}$ and $\mu_{\bt^{\l}}^{f} \neq \mu_{\bt^{\l}}^{b}$ in lines \ref{algo:ana_get_noise_s}-\ref{algo:ana_get_noise_e}.
To distinguish this variant of ANA from the previous one, we will refer to the former with the term synchronous ANA.

\section{Experimental evaluation}
\label{sec:experiments}
In Section~\ref{sec:algorithm}, we remarked that the measures $\mu_{\om^{\l}}^{f}$ and $\mu_{\bt^{\l}}^{f}$ need to collapse onto Dirac's deltas in order to yield a QNN at inference time.
It is known that a uniform distribution $\mathcal{U}[a, b]$ over a non-empty real interval $[a, b]$ can also be described as $\mathcal{U}[\frac{a+b}{2} - \sqrt{3}\sg, \frac{a+b}{2} + \sqrt{3}\sg]$, where $\sg$ is its standard deviation.
Modelling $\sg = \sg(t)$ as a time-dependent quantity, we see that the zero-mean distribution
\begin{equation}\label{eq:zeromean_uniform}
  \mathcal{U}[-\sqrt{3}\sg(t), \sqrt{3}\sg(t)]
\end{equation}
can be collapsed to $\delta_{0}$ if $\sg(t) \to 0$ as $t \to 0$.
We thus modelled $\mu_{\om^{\l}}^{f}, \mu_{\om^{\l}}^{b}, \mu_{\bt^{\l}}^{f}$ and $\mu_{\bt^{\l}}^{b}$ as products of independent uniform distributions like \eqref{eq:zeromean_uniform}.
In other words, we added a zero-mean uniform noise of given standard deviation to each parameter.
The marginal distributions of these product measures were independent by definition, but not necessarily identically distributed (different parameters could be added noises with different standard deviations).
To anneal these measures to Dirac's deltas, we used their component's standard deviations as time-dependent hyperparameters.
To avoid the exploding gradients problem predicted by Theorem~\ref{th:regularization}, we heuristically opted for hierarchical annealing of the noise, starting from layer $\ph_{\m^{1}}$ to layer $\ph_{\m^{L}}$.
All the weights and the activation functions of our models were set to be ternary, with thresholds $\Theta = \{-0.5, +0.5\}$ and quantization levels $Q = \{-1, 0, +1\}$.
The weights were initialized uniformly around the thresholds.
We ran all the experiments for $1000$ epochs on machines equipped with an Intel Xeon E5-2640v4 CPU, four Nvidia GTX1080\,Ti GPUs, and 128\,GB of memory.

\paragraph{CIFAR-10}
\begin{wrapfigure}{R}{.5\textwidth}
  \centering
  \includegraphics[width=.5\textwidth]{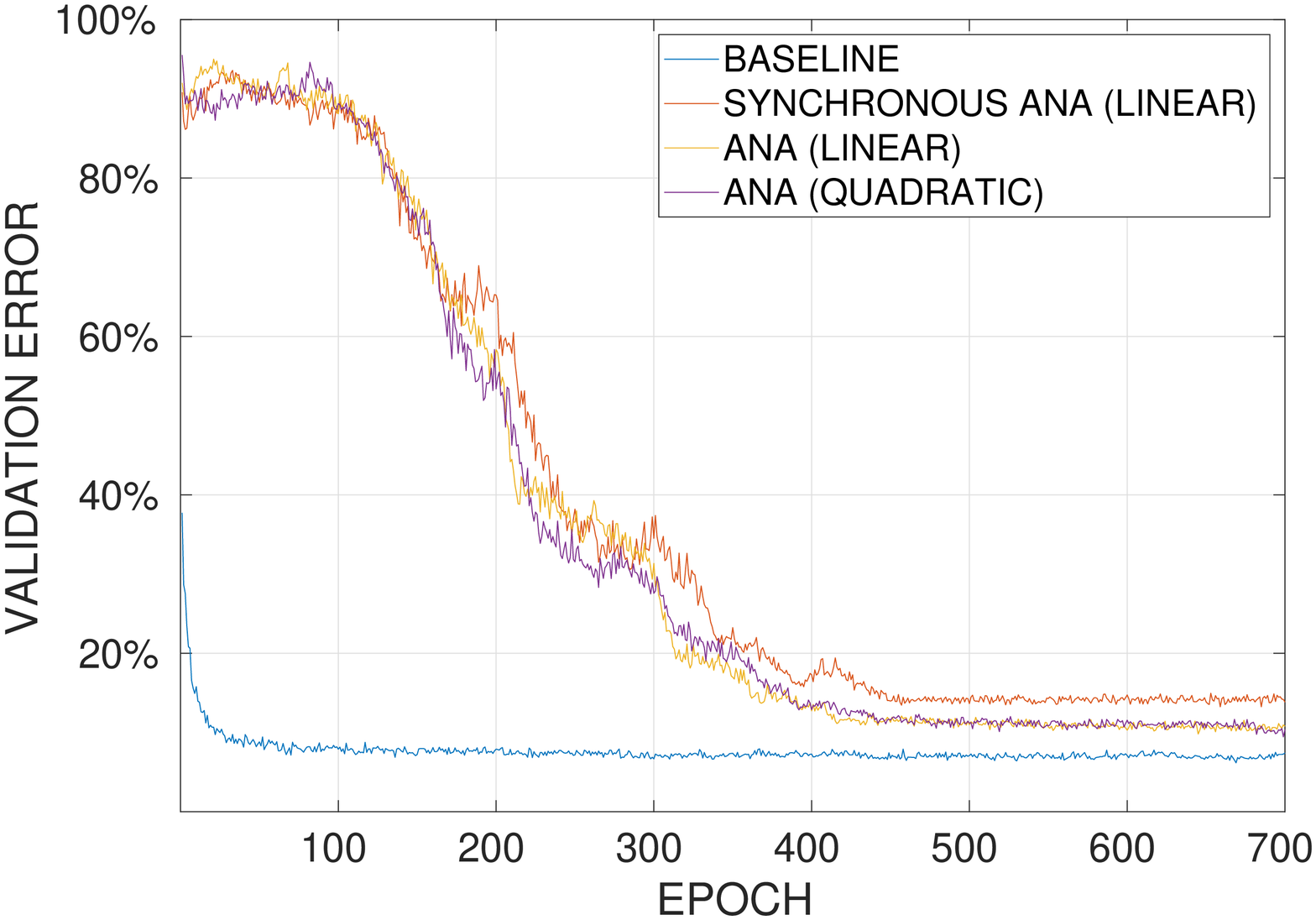}
  \caption{Validation error of the VGG-like CNNs trained on CIFAR-10. The long delay required to start convergence is because validation inference was always performed removing the noise from all the layers, despite its presence during training.}
  \label{fig:cifar10_valid}
\end{wrapfigure}
We used the per-class hinge loss also used by \cite{Hubara2018}, and set a manual learning rate decay schedule in combination with the Adam optimization algorithm \cite{Kingma2014}.
In all our experiments, and for the sake of comparison with related work, we used the same VGG-like architecture used by \cite{Hubara2018} and \cite{Deng2018}, consisting of six convolutional layers followed by three fully connected layers.
The first experiment used the synchronous version of Algorithm~\ref{algo:ana}.
For each layer, we defined a noise decay period of $50$ epochs, during which the standard deviation was reduced linearly from a given starting value down to zero.
This decay period was meant to allow the representations of the corresponding layer to stabilize, before starting the quantization of the following map.
Since the noise has to be removed to get a quantized network, the limitation of synchronous ANA is that it prevents gradients from flowing through multi-step functions and adjusting the lower layers' parameters.
To circumvent this problem, we used forward measures $\mu_{\om^{\l}}^{f}$ and $\mu_{\bt^{\l}}^{f}$ whose controlling standard deviations $\sg_{\om^{\l}}^{f}, \sg_{\bt^{\l}}^{f}$ were linearly or quadratically annealed to zero, and constant backward measures $\mu_{\om^{\l}}^{b}$ and $\mu_{\bt^{\l}}^{b}$ such that $\sg_{\om^{\l}}^{b} = \sg_{\bt^{\l}}^{b} \equiv \bar{\sg} \neq 0$.
Since the network was able to continue training even after the noise had been removed during the forward pass, both the linear and quadratic settings outperformed synchronous annealing, showing no relevant mutual difference.
Validation errors are reported in Figure~\ref{fig:cifar10_valid}.
Validation accuracies are reported in Table~\ref{tab:experiments}.
The exact experimental settings are reported in Appendix~\ref{app:experiments}.

\paragraph{ImageNet}
We used the cross-entropy loss function and a manual learning rate decay schedule in combination with the Adam optimization algorithm.
We experimented on AlexNet \cite{Krizhevsky2012} and MobileNetV2 \cite{Sandler2018}.
A dilemma that arises when quantizing networks that use residual connections, such as MobileNetV2, is how to deal with the algebra between two representation spaces $X^{\l_{1}}$ and $X^{\l_{2}}$ that are joined after a residual branch: in which space should the operation $\x^{\l_{1}} + \x^{\l_{2}}$ take values?
The definition of a suitable algebra satisfying a closure property is not the focus of the present work.
We thus opted for quantizing the residual branches while keeping the bottleneck layers at full-precision (see also \cite{Liu2018}).
Since the last layers of the residual branches are affine transformations, these layers were weight-$Q$-quantized.
We also quantized the input tensors to all the residual branches to ensure the computationally costly convolution operations are performed with ternary arguments.
Although this leads to an only partially quantized model, the operations performed in the residual branches of a MobileNetV2 amount to approximately $70\%$ of the total operations, consequently yielding a significant reduction in computational effort.
Based on the findings from the CIFAR-10 experiments, we used linearly decaying standard deviations in the forward pass and constant non-zero standard deviations in the backward pass, until full quantization was reached, for both AlexNet and MobileNetV2.
Validation accuracies are reported in Table~\ref{tab:experiments}.
The exact experimental settings are reported in Appendix~\ref{app:experiments}.
\begin{table}[ht]
  \caption{Top-1 validation accuracy of ANA on both CIFAR-10 and ImageNet datasets. Quantized models also report the relative accuracy with respect to the corresponding baselines.}
  \label{tab:experiments}
  \centering
  \begin{tabular}{llcc} \toprule
    \multirow{2}{*}{Problem} & \multirow{2}{*}{Network Topology} & \multicolumn{2}{c}{Top-1 Accuracy} \\
                             &                          & Absolute & Relative \\ \midrule
    \multirow{2}{*}{CIFAR-10} & VGG-like Baseline       & 94.40\%  & -       \\
                              & VGG-like                & 90.74\%  & 96.12\% \\ \midrule
    \multirow{4}{*}{ImageNet} & AlexNet Baseline        & 54.71\%  & -       \\
                              & AlexNet                 & 45.80\%  & 83.71\% \\
                              & MobileNetV2 Baseline    & 71.25\%  & -       \\
                              & MobileNetV2 (Residuals) & 64.79\%  & 90.93\% \\ \bottomrule
  \end{tabular}
\end{table}

Comparisons between ANA and similar low-bitwidth (BNNs, TNNs) quantization methods are reported in Table~\ref{tab:comparison}.
We evaluated our method on AlexNet to be able to compare with other algorithms, where we achieved a validation accuracy of 45.8\%.
In comparison to the most widely known STE (41.8\% accuracy) \cite{Hubara2018}, this shows a clear reduction of the accuracy gap to the full-precision network (54.7\% accuracy) by 31\% when training it as a TNN with ANA.
Comparing to a BNN trained with XNOR-Net, our trained TNN generally had a richer configuration space and the achieved top-1 accuracy gain is less distinct.
However, note that XNOR-Net does not quantize the first and last layers, and additional normalization maps have to be computed \cite{Rastegari2016}.
To analyze the application of TNNs to a more recent and compute-optimized network, we evaluated the accuracy of MobileNetV2 with quantized residuals.
This simplifies 70\% of the network's convolution operations to work on ternary operands while introducing an accuracy loss of only 6.46\% from 71.25\% to 64.79\%.
To perform a more direct comparison to a TNN trained with GXNOR-Net \cite{Deng2018}, we evaluate our method on CIFAR-10 with a VGG-like network.
Here we observe a 1.76\% inferior accuracy; however, the scalability of the GXNOR-Net method to deeper networks remains unclear.
Furthermore, we think an exploration of initialization strategies for ANA could further improve the resulting accuracy.
\begin{table}[ht]
  \caption{Comparison between the Top-1 validation accuracies of ANA and similar methods described in QNN literature.}
  \label{tab:comparison}
  \centering
  \begin{threeparttable}
    \begin{tabular}{lccccc} \toprule
      \multirow{2}{*}{Algorithm} & \multirow{2}{*}{Type} & \multirow{2}{*}{\makecell{1st/Last Layers\\Quantized\tnote{1}}} & CIFAR-10 & \multicolumn{2}{c}{-------- ImageNet --------} \\                                 &     &                       & VGG-like & AlexNet & MobileNetV2\tnote{2} \\ \midrule
      STE \cite{Hubara2018}          & BNN & \ding{51} / \ding{51} & 89.85\%  & 41.80\% & -       \\ 
      XNOR-Net \cite{Rastegari2016}  & BNN & \ding{55} / \ding{55} & -        & 44.20\% & -       \\ 
      GXNOR-Net \cite{Deng2018}      & TNN & \ding{51} / \ding{51} & 92.50\%  & -       & -       \\ 
      \textbf{ANA (Ours)}            & TNN & \ding{51} / \ding{51} & 90.74\%  & 45.80\% & 64.79\% \\ \bottomrule
    \end{tabular}
    \begin{tablenotes}
    \item[1] The last layer is actually weight-$Q$-quantized, since it implements an affine map.
    \item[2] Only the residual branches of MobileNetV2 were quantized ($\sim$70\% of MACs).
    \end{tablenotes}
  \end{threeparttable}
\end{table}

\section{Conclusions}
\label{sec:conclusion}
We showed that QNNs are dense in the space of continuous functions defined on a hypercube.
We then investigated the role of probability in gradient-based training of this family of learning machines.
The idea of variational inference is not new in DL \cite{Hinton1993, Graves2011}.
Our original application of probability to DL is the description of the regularization effect played by noise on multi-step functions (Theorem~\ref{th:regularization}).
We established that applying the expectation operator to random functions can return differentiable functions, and applied this general result to the design of a new algorithm to train QNNs (ANA).
In particular, the use of uniform noise allowed an efficient implementation of ANA on top of existing DL frameworks.
We were able to achieve state-of-the-art quantization results applying TNNs to CIFAR-10 and ImageNet, also on the mobile-friendly topology MobileNetV2.

\bibliography{biblio}
\bibliographystyle{ieeetr}

\clearpage
\appendix

\section{Proofs}
\label{app:proofs}
\begin{lemma}\label{th:characteristic}
Let $n_{0} > 0$ be a given integer and let $P$ be a hyperparallelepiped in $\R^{n_{0}}$ defined as the cartesian product of bounded intervals $I_{1}, \dots, I_{n_{0}}$.
Then its characteristic function $\chi_{P}(\x^{0})$ can be represented as a ternary FNN.
\end{lemma}
\begin{proof}
We prove the specific case where the intervals $I_{i} = [p_{i}, q_{i}]$ are closed.
\begin{equation*}
  \chi_{P}(\x^{0}) = \chi_{\w, \b}(\x^{0}) = \sigma \left( \langle \sigma(\w^{1} \cdot \x^{0} + \b^{1}), \w^{2} \rangle + b^{2} \right) \,,
\end{equation*}
where $\w^{1} \in \{-1, 0, 1\}^{2n_{0} \times n_{0}}, \b^{1} \in \R^{2n_{0}}, \w^{2} = \one_{2n_{0}}$ and $b^{2} = -2n_{0}$.
In particular, $\w^{1}, \b^{1}$ are defined by
\begin{equation*}
  w^{1}_{ij} =
  \begin{cases}
    \delta_{ij}, &\text{if} \,\, 1 \le i \le n_{0} \\
    -\delta_{(i-n_{0})j}, &\text{if} \,\, n_{0}+1 \le i \le 2n_{0}
  \end{cases}
  \quad\text{and}\quad
  b^{1}_{i} =
  \begin{cases}
    p_{i}, &\text{if} \,\, 1 \le i \le n_{0} \\
    q_{i-n_{0}}, &\text{if} \,\, n_{0}+1 \le i \le 2n_{0} \,,
  \end{cases}
\end{equation*}
for $i = 1, \dots, 2n_{0}$ and $j = 1, \dots, n_{0}$, where $\delta_{ij}$ denotes the Kronecker delta.
Let us observe that $\chi_{\w, \b}(\x^{0}) = 1$ if and only if the argument of the outer activation $\sigma$ is non-negative, that is,
\begin{equation*}
  \langle \sigma(\w^{1} \cdot \x^{0} + \b^{1}), \w^{2} \rangle \ge -b^{2} \,,
\end{equation*}
which is satisfied if and only if one has equality in the inequality above, which means
\begin{equation*}
  \sigma(\w^{1} \cdot \x^{0} + \b^{1}) = \w^{2} \,.
\end{equation*}
The latter vectorial equation corresponds to
\begin{equation*}
  \sum_{j = 1}^{n_{0}} w^{1}_{ij} x^{0}_{j} + b^{1}_{i} \ge 0, \,\forall\, i = 1, \dots, 2n_{0} \,,
\end{equation*}
that is, by the definitions of $w^{1}_{ij}$ and $b^{1}_{i}$,
\begin{equation*}
  p_{j} = -b^{1}_{j} \le x^{0}_{j} \le b_{j+n_{0}} = q_{j}, \,\forall\, j = 1, \dots, n_{0} \,.
\end{equation*}
These last inequalities are satisfied if and only if $\x^{0} \in P$.
We observe that if the intervals $I_{1}, \dots, I_{n_{0}}$ in the definition of $P$ are replaced by half-open or even open intervals, the first layer could still detect the parallelepiped $P$.
In fact, it would suffice to replace the suitable activations $\sigma^{1}_{i}, i = 1, \dots, 2n_{0}$ with the Heaviside $\sigma^{-}$ that takes value zero at zero:
\begin{equation*}
  \sigma^{-}(x) = 
  \begin{cases}
    0, &\text{if} \,\, x \leq 0 \\
    1, &\text{if} \,\, x > 0
  \end{cases} \,.
\end{equation*}
This substitution does not alter the quantized structure of the network.
Hence the lemma is proved.
\end{proof}
The fundamental idea of the lemma is thus to interpret the hyperparalleliped $P = \cap_{j = 1}^{n_{0}} \left( \{ \x^{0} | x^{0}_{j} \geq p_{i} \} \cap \{ \x^{0} | x^{0}_{j} \leq q_{i} \} \right)$ as the intersection of closed half-spaces.
A filter $(\w^{1}_{i}, b^{1}_{i})$ with ternary weights can thus be used to measure whether $\x^{0}$ belongs to the corresponding half-space.

\subsection*{Proof of Theorem~\ref{th:approximation}}
\begin{proof}
We start by explicitly constructing a ternary neural network that can represent a function $f$ which is constant on hyper-parallelepipeds.
Let $N$ be a positive integer.
Let $\{ P_{1}, \dots, P_{N} \}$ be a family of closed hyper-parallelepipeds such that $X^{0} = \bigcup_{s=1}^{N} P_{s}$ and $P_{s_{1}} \cap P_{s_{2}} = \emptyset$ (i.e., $\{P_{s}\}_{s = 1, \dots N}$ is a partition of $X^{0}$).
Set now $n_{1} = 2n_{0} N$ and define $\w^{1} = (\w^{1}_{1}, \dots, \w^{1}_{N}), \b = (\b^{1}_{1}, \dots, \b^{1}_{N})$, where $\w^{1}_{s} \in \{-1, 0, 1\}^{2n_{0} \times n_{0}}$ and $\b^{1}_{s} \in \R^{2n_{0}}$ for $s = 1, \dots, N$.
Define
\begin{align*}
  \ph_{\m^{1}} \,:\,
  \R^{n_{0}} &\to \{0,1\}^{N} \\
  \x &\mapsto \left( \chi_{\w^{1}_{1}, \b^{1}_{1}}(\x), \dots, \chi_{\w^{1}_{N}, \b^{1}_{N}}(\x) \right) \,,
\end{align*}
where $\chi_{\w^{1}_{s}, \b^{1}_{s}}$ is the ternary network representation of the half-spaces enclosing $\chi_{P_{s}}$ (i.e., the first layer described in Lemma~\ref{th:characteristic} but applied to every $P_{s}$ in parallel).
Now define $\w^{2} \in \{ 0, 1 \}^{N \times 2n_{0}N}$ and $\b^{2} \in \R^{N}$ such that
\begin{equation*}
  w^{2}_{ij} =
  \begin{cases}
    1, &\text{if} \,\, 2n_{0}(i-1) + 1 \leq j \leq 2n_{0}i \\
    0, &\text{otherwise}
\end{cases}
\end{equation*}
and
\begin{equation*}
  b^{2}_{i} = -2n_{0}
\end{equation*}
for all $i = 1, \dots N$.
The map $\ph_{m^{2}} \circ \ph_{m^{1}}$ thus measures the membership of a point $\x^{0}$ to the hyper-parallelepipeds $P_{s}$.
Since $\{P_{s}\}_{s = 1, \dots N}$ partitions $X^{0}$, just one neuron can be active at a time.
Finally, for a given $\w^{3} \in \R^{N}$ we define the affine map
\begin{equation*}
  \ph_{\m^{3}}(\x^{2}) = \langle \w^{3}, \x^{2} \rangle \,.
\end{equation*}
Finally, we set
\begin{equation}
  \Phi_{\hm}(\x^{0}) = \ph_{\m^{3}} \circ \ph_{\m^{2}} \circ \ph_{\m^{1}} (\x^{0}) \,.
\end{equation}
Let now $f \in \Lip_{\lambda}(X^{0})$ be fixed.
Let $n$ be an integer that satisfies
\begin{equation*}
  n \ge \frac{2 \sqrt{n_{0}} S \lambda}{\epsilon} \,,
\end{equation*}
then choose $N = n^{n_{0}}$.
Consider the family of closed hypercubes $P_{s}$ with side length $\delta = S/n$, forming a partition $\{ P_{s} \}_{s = 1, \dots, N}$ of $X^{0}$.
For every $s = 1, \dots, N$ we can identify the hypercube $P_{s}$ by the index tuple $(i^{s_{0}}, \dots, i^{s_{n_{0}-1}})$ whose $n_{0}$ components are the unique integers $i^{s_{k}}\in \{ 0, \dots, n_{0}-1 \}$ such that 
\begin{equation*}
  s-1 = i^{s_{0}} + i^{s_{1}} n + i^{s_{2}} n^{2} + \dots + i^{s_{n_{0}-1}} n^{n_{0}-1} \,.
\end{equation*}
Then the hypercube $P_{s}$ is given by
\begin{equation*}
  P_{s} = [i^{s_{0}}\delta, (i^{s_{0}}+1)\delta] \times \dots \times [i^{s_{n_{0}-1}}\delta, (i^{s_{n_{0}-1}}+1)\delta]\,.
\end{equation*}
Define $w^{3}_{s}$ as the integral average of $f$ on $P_{s}$ for each $s = 1, \dots, N$, so that $\w^{3} = (w^{3}_{1}, \dots, w^{3}_{N})$. Define $\Phi_{\hm}$ as above, but applying Lemma \ref{th:characteristic} with $P = P_{s}$ for $s = 1, \dots, N$.
We are now left with showing that 
\begin{equation}\label{eq:approximation}
  |\Phi_{\hm}(\x^{0}) - f(\x^{0})| \le \epsilon \,\, \forall \, \x^{0} \in X^{0} \,.
\end{equation} 
Let $s \in \{ 1, \dots, N \}$ be such that $\x^{0} \in P_{s}$.
Then $\Phi_{\hm}(\x^{0}) = \ph_{\m^{3}}(\ph_{\hm^{2}}(\x^{0})) = w^{3}_{s} = f(\x^{0}_{s})$ for some $\x^{0}_{s} \in P_{s}$, hence by the Lipschitz property of $f$ we obtain
\begin{equation}\label{eq:estimate}
  |\Phi_{\hm}(\x^{0}) - f(\x^{0})| = |f(\x^{0}_{s}) - f(\x^{0})| \leq \lambda |\x^{0}_{s} - \x^{0}| \leq L \diam(P_{s}) = \lambda\sqrt{n_{0}}S/n \le \epsilon/2 < \epsilon \,.
\end{equation}
Since \eqref{eq:estimate} holds for every $\x^{0} \in X^{0}$, we obtain \eqref{eq:approximation}, as wanted.
\end{proof}

\subsection*{Proof of Theorem~\ref{th:regularization}}
We first introduce the essential definitions and notation needed to understand the statement of Theorem~\ref{th:regularization} and its proof.
For further details, see for instance \cite{Ambrosio2000}.
Given a measurable function $f \,:\, \R \to \R$, we say that $f \in L^{p}(\R)$ for $1 \le p < \infty$ if its $p$-norm $\|f\|_p = \left( \int_{\R} |f(x)|^{p} dx \right)^{1/p}$ is finite.
We say that $f \in L^{\infty}(\R)$ if its $\infty$-norm $\|f\|_\infty$, defined as the infimum of the values $t > 0$ such that the set $\{ x \in \R | |f(x)| > t \}$ has zero Lebesgue measure, is finite.
Given $f \in L^{1}(\R)$, we call \textit{distributional derivative} of $f$ the (linear, continuous) functional $Df$ defined by
\begin{equation*}
  Df(\phi) := -\int_{\R} f(x) \phi'(x) dx
\end{equation*}
on any $C^{\infty}$-smooth test function $\phi$ that vanishes outside some compact interval of $\R$.
When $Df$ is represented by a function, that is, there exists $g \in L^{1}(\R)$ such that we have the integration by parts formula
\begin{equation*}
  Df(\phi) = \int_{\R} g(x) \phi(x) dx\,,
\end{equation*}
we say that $f$ belongs to the \textit{Sobolev space} $W^{1,1}(\R)$, and that $Df = g$ is the \textit{weak derivative} of $f$.
Similarly, but more generally, when $Df$ is represented via integration by parts by a signed Borel-regular measure on $\R$, whose total variation $|Df|(\R)$ (roughly speaking, a generalization of the $L^{1}$-norm of the derivative of $f$) is finite, then we say that $f$ is a function of \textit{bounded variation}, i.e., that it belongs to the space $BV(\R)$.
For instance, the characteristic function $\chi_{(a,b)}$ of a bounded interval $(a,b) \subset \R$ is a BV function, and its distributional derivative $D\chi_{(a,b)}$ is given by the signed measure $\delta_{a} - \delta_{b}$, where $\delta_{q}$ denotes the Dirac's delta measure centered at $q$.
In order to provide an interpretation of STE, and in accordance with the choices of the noise in our experiments, we assume the noise density $\mu$ to be either a Sobolev or, more generally, a BV function on $\R$.
Hence, its distributional derivative $D\mu$ will be either an $L^{1}$ function or a signed measure with finite total variation.
Of course, Theorem~\ref{th:continuity} includes the special case when $\mu$ is smooth.
\begin{proof}[Proof of Theorem \ref{th:regularization}]
The first claim (i) simply follows from the change of variables $t = x + \nu$:
\begin{align*}
  \E_{\mu}[\sigma_{\nu}](x)
  &= \int_{\R} \sigma(x +\nu) \mu(-\nu) d\nu = \int_{\R} \sigma(t) \mu(x-t) dt = (\sigma \ast \mu)(x), \,\forall\, x \in \R\,.
\end{align*}
For the second claim (ii), we fix a test function $\phi \in C^{\infty}_{c}(\R)$ and note that
\begin{align*}
  (\E[D\sigma_{\nu}], \phi)
  &= \E[(D\sigma_{\nu}, \phi)]
  &= -\int_{\R} \int_{\R} \sigma_{\nu}(x) \phi'(x) dx \mu(-\nu) d\nu \,.
\end{align*}
By Fubini's Theorem we can exchange the order of integration and get
\begin{align}\label{eq:fubini}
  (\E[D\sigma_{\nu}], \phi)
  &= -\int_{\R} \sigma \ast \mu(x) \phi'(x) dx \,,
\end{align}
which shows the first equality $\E[D\sigma_{\nu}] = D\E[\sigma_{\nu}]$.
Then, being $\mu$ a Sobolev function easily implies that $\sigma \ast \mu$ is continuously differentiable and one has $\frac{d}{dx} \sigma \ast \mu = \sigma \ast D\mu$, hence integrating by parts in \eqref{eq:fubini} gives the second equality $D\E[\sigma_\nu] = \sigma \ast D\mu$.
The third, and last, claim (iii) follows from similar computations as in the proof of (ii).
\end{proof}
\begin{figure}[ht]
  \centering
  \includegraphics[width=\textwidth]{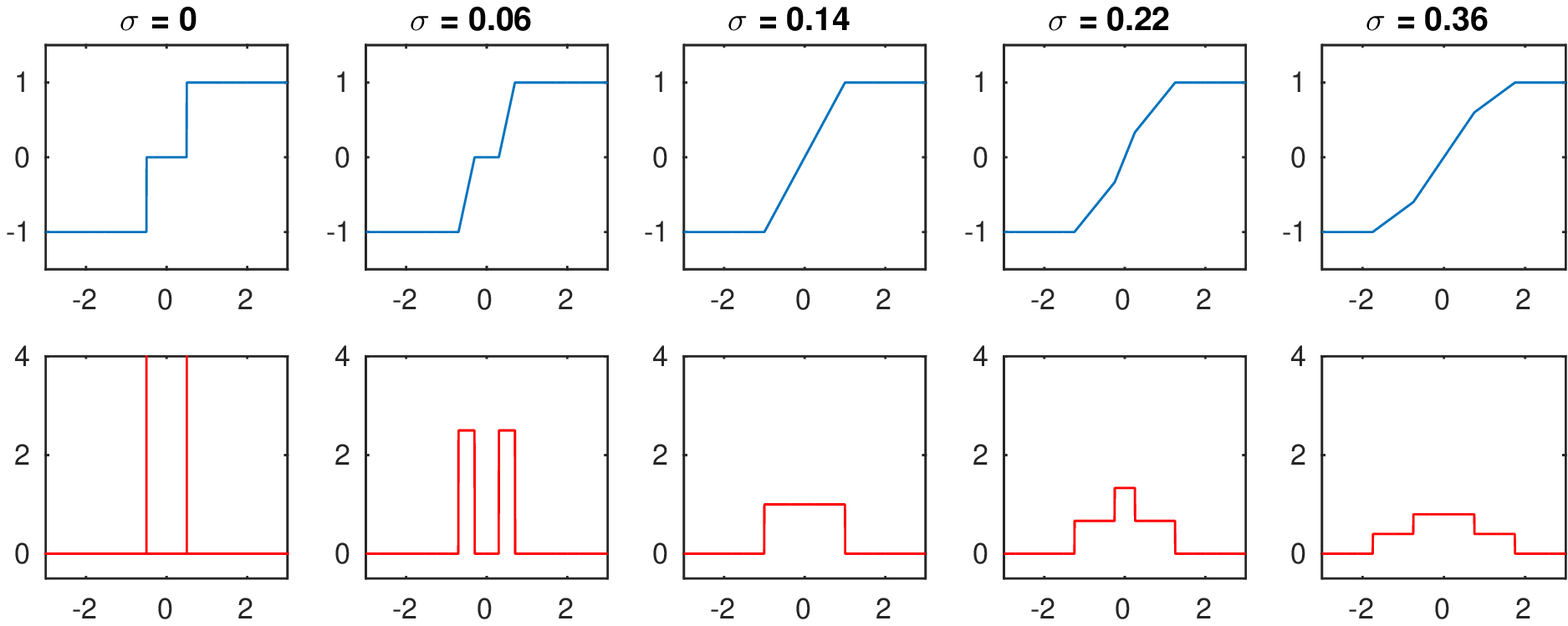}
  \caption{Representation of a regularized quantization function (top) and its derivative (bottom) when uniform additive noise (with different standard deviations) is applied to the argument. The multistep function depicted is a ternary function with $\Theta = \{-0.5, +0.5\}$ and $Q = \{-1, 0, +1\}$.}
  \label{fig:uniform}
\end{figure}
\begin{figure}[ht]
  \centering
  \includegraphics[width=\textwidth]{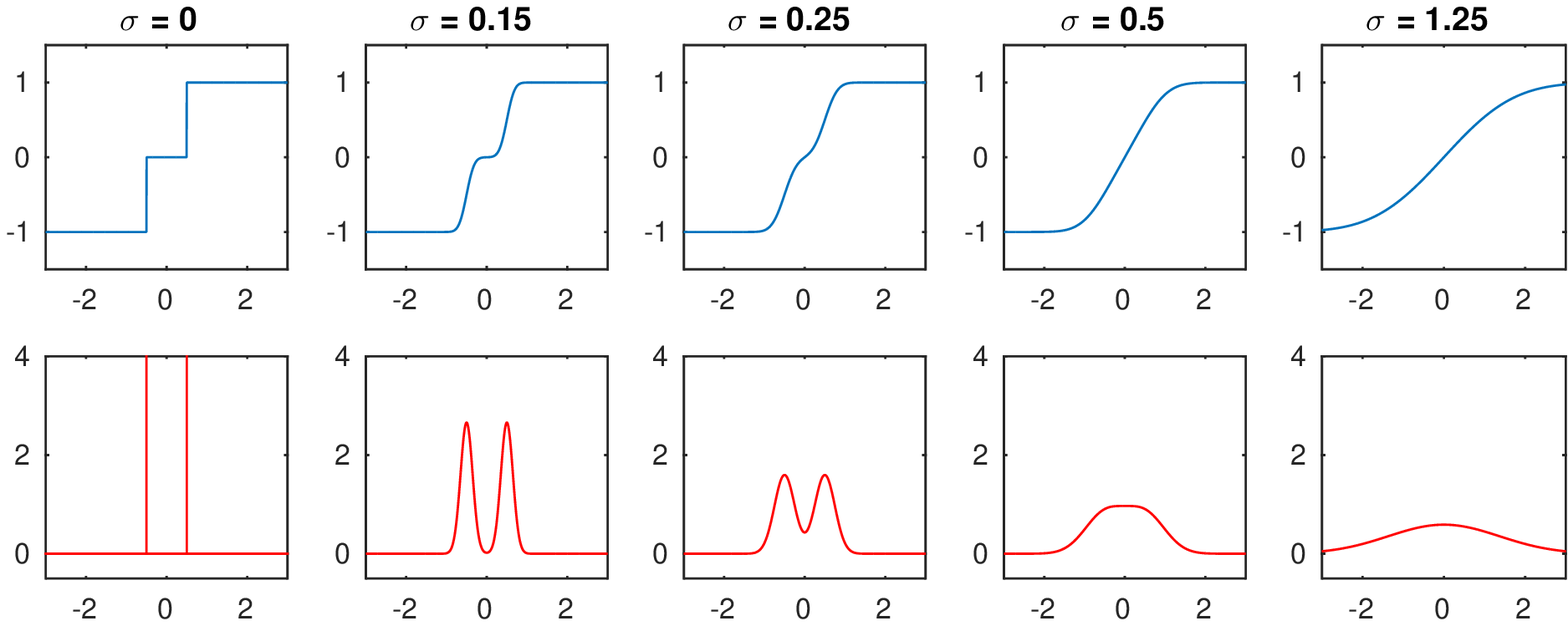}
  \caption{Representation of the regularized quantization function (top) and its derivative (bottom) when Gaussian additive noise (with different standard deviations) is applied to the argument. The multistep function depicted is a ternary function with $\Theta = \{-0.5, +0.5\}$ and $Q = \{-1, 0, +1\}$.}
  \label{fig:gaussian}
\end{figure}

\subsection*{Proof of Theorem~\ref{th:continuity}}
It is worth recalling the notion of convergence for a (probability) measure that is referred to in the theorem.
We say that a sequence $\mu_{t}$ of Borel measures restricted to a compact subset $\Xi$ of the Euclidean $n$-space converges weakly-$*$ to a Borel measure $\mu$ on $\Xi$ if, for every continuous function $\phi \,:\, \Xi \to \R$, one has $\E_{\mu_{t}}[\phi] \to \E_{\mu}[\phi]$ as $t \to \infty$.
We adopt hereafter the general notation \eqref{eq:composed_map} for a parametric composition of maps, as Theorem~\ref{th:continuity} applies to this situation as well.
We recall in particular the map $\Psi_{\hm^{\l}} \,:\, X^{0} \to X^{\l}$ defined as $\Psi_{\hm^{\l}} = \psi_{\m^{\l}} \circ \cdots \circ \psi_{\m^{1}}$, as it will play a role in the induction argument below.
\begin{proof}[Proof of Theorem~\ref{th:continuity}]
First of all we observe that the continuity assumption on $\psi^{\l}$ implies the existence of a modulus of continuity $\eta$ (i.e., $\eta \,:\, [0, +\infty) \to [0, +\infty)$ is continuous, strictly increasing, and satisfies $\eta(0) = 0$) such that  
\begin{equation}\label{eq:unifcont}
  \sup_{\xx^{\l} \in \Xi^{\l}} |\psi_{\xx^{\l}}(\x) - \psi_{\xx^{\l}}(\y)| \le \eta(|\x-\y|), \,\forall\, \x, \y \in X^{\l-1}, \,\forall\, \l = 1, \dots, L \,.
\end{equation}
Denote by $\E_{\l, t}$ the expectation operator associated with the measure $\mu^{\l}_{t}$.
We proceed by induction on $\l = 1, \dots, L$.
The basis of the induction consists in showing that 
\begin{equation}\label{eq:inductionbasis}
  \lim_{t \to \infty} \E_{1, t}[\psi_{\xx^{1}}](\x) = \psi_{\m^{1}}(\x), \,\forall\, \x \in X^{0} \,.
\end{equation}
We observe that
\begin{equation*}
  \E_{1, t}[\psi_{\xx^{1}}](\x) = \int_{\Xi^{1}} \psi(\xx^{1}, \x) d\mu^{1}_{t}(\xi^{1}) \,,
\end{equation*}
hence \eqref{eq:inductionbasis} directly follows from the definition of convergence of $\mu^{1}_{t}$ to $\delta_{\m^{1}}$.
In the next, inductive step we shall apply the uniform continuity assumption \eqref{eq:unifcont} in an essential way.
Let us set 
\begin{equation*}
  J_{\l, t}(\x) = \E_{\l, t}[\psi_{\xx^{\l}}] \circ \cdots \circ \E_{1, t}[\psi_{\xx^{1}}](\x)
\end{equation*} 
and assume by induction that 
\begin{equation}\label{eq:inductivehp}
  \lim_{t \to \infty} J_{\l, t}(\x) = \Psi_{\hm^{\l}}(\x) \,\forall\, \x \in X^{0} \,.
\end{equation}
We then have to prove that, for some $1 \le \l < L$ and for all $\x \in X^{0}$,
\begin{equation}\label{eq:inductivestep}
  \lim_{t \to \infty} \E_{\l+1, t}[\psi_{\xx^{\l+1}}] \circ J_{\l, t}(\x) = \psi_{\m^{\l+1}} \circ \Psi_{\hm^{\l}}(\x) \,.
\end{equation}
Let us estimate
\begin{align*}
  \Big| \E_{\l+1, t}[\psi_{\xx^{\l}}] \circ J_{\l, t}(\x) - \psi_{\m^{\l+1}} \circ \Psi_{\hm^{\l}}(\x) \Big|
  &= \left| \int_{\Xi^{\l+1}} \Big( \psi_{\xx^{\l+1}}(J_{\l, t}(\x)) - \psi_{\m^{\l+1}}(\Psi_{\hm^{\l}}(\x)) \Big) d\mu^{\l+1}_{t}(\xx^{\l+1}) \right| \\
  &\le \eta(|J_{\l, t}(\x) - \Psi_{\hm^{\l}}(\x)|) \,,
\end{align*}
where the last inequality follows from \eqref{eq:unifcont} and the fact that $\mu^{\l+1}_{t}$ is a probability measure.
It is then immediate to see that the last term of the previous estimate goes to zero by \eqref{eq:inductivehp}.
This proves \eqref{eq:inductivestep} and completes the proof of the theorem. 
\end{proof}

\clearpage
\section{Experimental setup}
\label{app:experiments}
\subsection*{CIFAR-10}
CIFAR-10 consists of $32 \times 32$ pixels RGB images grouped into ten classes.
It comprises $50k$ training points and $10k$ test points.
For our experiments, we split the $50k$ images training set in a $45k$ images actual training set and a $5k$ validation set.
We performed data augmentation by resizing, random cropping and random flipping.
The resizing and random cropping were implemented using PyTorch's \texttt{torchvision.transforms.RandomCrop} function with parameter \texttt{padding=4}.
The resulting image then was randomly flipped using PyTorch's \texttt{torchvision.transforms.RandomHorizontalFlip}.
The preprocessing consisted of a normalization of the three channels described by means $\mu = (0.4914, 0.4822, 0.4465)$ and standard deviations $\sg = (0.2470, 0.2430, 0.2610)$.
The learning rate was initialized to $0.001$ and decreased to $0.0001$ at epoch $700$.
The batch size was set to $256$.
In the synchronous setting, the standard deviations regulating the noises of the layers $\ph_{\xx^{\l}}, \l = 1, \dots, L$ were initialized to $\sg_{\om^{\l}}^{f} = \sg_{\om^{\l}}^{b} = \sg_{\bt^{\l}}^{f} = \sg_{\bt^{\l}}^{b} = \sg_{\l} = \sqrt{3}/6$ (in order to describe the uniform distribution $\mathcal{U}[-1, +1]$) and annealed following the linear decay
\begin{equation*}
  \sg_{\l} = 1 - \frac{\min(\max(0, t-50(\l-1)), 50}{50} \,,
\end{equation*}
where $t$ represents the training epoch.
The idea was that $\sg_{\l}$ should decay from $\sqrt{3}/6$ at epoch $50(\l-1)$ to zero at epoch $50\l$.
In the asynchronous setting, the standard deviations regulating the noises were again initialized to $\sg_{\om^{\l}}^{f} = \sg_{\om^{\l}}^{b} = \sg_{\bt^{\l}}^{f} = \sg_{\bt^{\l}}^{b} = \sqrt{3}/6$.
We annealed the forward noises' standard deviations trying both, the already described linear decay and the quadratic decay
\begin{equation*}
  \sg_{\l} = \left( 1 - \frac{\min(\max(0, t-50(\l-1)), 50}{50} \right)^{2} \,.
\end{equation*}
The standard deviations regulating the backward noises were kept constant.
\begin{table}[ht]
  \caption{The VGG-like network used for CIFAR-10 experiments. For each map, $n$ indicates the number of filters, $p$ the (symmetric) padding size, $k$ the (square) kernel spatial side and $s$ the stride in both dimensions.}
  \label{tab:vgg-like}
  \centering
  \begin{tabular}{|l|l|l|c|c|c|c|l|} \hline
    Layer & Input Shape & Type & n & p & k & s & Output Shape \\ \hline
    \hline
    \multirow{3}{2em}{$\ph^{1}$} & \multirow{3}{7em}{\threedtnsrshp{32}{32}{3}} & Conv2d      & 128 & 1 & 3 & 1 & \multirow{3}{7em}{\threedtnsrshp{32}{32}{128}} \\ \cline{3-7}
                                 & & BatchNorm2d & \multicolumn{4}{c|}{-} & \\ \cline{3-7}
                                 & & QuantAct    & \multicolumn{4}{c|}{-} & \\ \hline
    \multirow{4}{2em}{$\ph^{2}$} & \multirow{4}{7em}{\threedtnsrshp{32}{32}{128}} & Conv2d      & 128 & 1 & 3 & 1 & \multirow{4}{7em}{\threedtnsrshp{16}{16}{128}} \\ \cline{3-7}
                                 & & MaxPool2d   & - & 0 & 2 & 2 & \\ \cline{3-7}
                                 & & BatchNorm2d & \multicolumn{4}{c|}{-} & \\ \cline{3-7}
                                 & & QuantAct    & \multicolumn{4}{c|}{-} & \\ \hline
    \multirow{3}{2em}{$\ph^{3}$} & \multirow{3}{7em}{\threedtnsrshp{16}{16}{128}} & Conv2d      & 256 & 1 & 3 & 1 & \multirow{3}{7em}{\threedtnsrshp{16}{16}{256}} \\ \cline{3-7}
                                 & & BatchNorm2d & \multicolumn{4}{c|}{-} & \\ \cline{3-7}
                                 & & QuantAct    & \multicolumn{4}{c|}{-} & \\ \hline
    \multirow{4}{2em}{$\ph^{4}$} & \multirow{4}{7em}{\threedtnsrshp{16}{16}{256}} & Conv2d      & 256 & 1 & 3 & 1 & \multirow{4}{7em}{\threedtnsrshp{8}{8}{256}} \\ \cline{3-7}
                                 & & MaxPool2d   & - & 0 & 2 & 2 & \\ \cline{3-7}
                                 & & BatchNorm2d & \multicolumn{4}{c|}{-} & \\ \cline{3-7}
                                 & & QuantAct    & \multicolumn{4}{c|}{-} & \\ \hline
    \multirow{3}{2em}{$\ph^{5}$} & \multirow{3}{7em}{\threedtnsrshp{8}{8}{256}} & Conv2d      & 512 & 1 & 3 & 1 & \multirow{3}{7em}{\threedtnsrshp{8}{8}{512}} \\ \cline{3-7}
                                 & & BatchNorm2d & \multicolumn{4}{c|}{-} & \\ \cline{3-7}
                                 & & QuantAct    & \multicolumn{4}{c|}{-} & \\ \hline
    \multirow{4}{2em}{$\ph^{6}$} & \multirow{4}{7em}{\threedtnsrshp{8}{8}{512}} & Conv2d      & 512 & 1 & 3 & 1 & \multirow{4}{7em}{\threedtnsrshp{4}{4}{512}} \\ \cline{3-7}
                                 & & MaxPool2d   & - & 0 & 2 & 2 & \\ \cline{3-7}
                                 & & BatchNorm2d & \multicolumn{4}{c|}{-} & \\ \cline{3-7}
                                 & & QuantAct    & \multicolumn{4}{c|}{-} & \\ \hline
    \multirow{3}{2em}{$\ph^{7}$} & \multirow{3}{7em}{\onedtnsrshp{8192}} & FC & 1024 & \multicolumn{3}{c|}{-} & \multirow{3}{7em}{\onedtnsrshp{1024}} \\ \cline{3-7}
                                 & & BatchNorm1d & \multicolumn{4}{c|}{-} & \\ \cline{3-7}
                                 & & QuantAct    & \multicolumn{4}{c|}{-} & \\ \hline
    \multirow{3}{2em}{$\ph^{8}$} & \multirow{3}{7em}{\onedtnsrshp{1024}} & FC & 1024 & \multicolumn{3}{c|}{-} & \multirow{3}{7em}{\onedtnsrshp{1024}} \\ \cline{3-7}
                                 & & BatchNorm1d & \multicolumn{4}{c|}{-} & \\ \cline{3-7}
                                 & & QuantAct    & \multicolumn{4}{c|}{-} & \\ \hline
    \multirow{2}{2em}{$\ph^{9}$} & \multirow{2}{7em}{\onedtnsrshp{1024}} & FC & 10 & \multicolumn{3}{c|}{-} & \multirow{2}{7em}{\onedtnsrshp{10}} \\ \cline{3-7}
                                 & & BatchNorm1d & \multicolumn{4}{c|}{-} & \\ \hline
  \end{tabular}
\end{table}

\subsection*{ImageNet}
ImageNet consists of $224 \times 224$ pixels RGB images grouped into $1000$ classes.
It comprises $1.2M$ training points and $50k$ validation points.
We performed data augmentation using a pipeline of random resizing and cropping, random flipping, random colours alterations and random PCA-based lighting changes.
The random cropping was implemented using PyTorch's \texttt{torchvision.transforms.RandomResizedCrop} function with default parameters.
Then, this new image underwent a random flipping using PyTorch's \texttt{torchvision.transforms.RandomHorizontalFlip}.
The colour alterations consisted of random changes of brightness, contrast and saturation (all are linear interpolations between the original image and transformations of its greyscale version).
Finally, the lighting changes were performed by random scaling of the three RGB channels, with coefficients depending on the eigenvalues obtained applying a per-pixel PCA on the ImageNet dataset.
The preprocessing consisted of a normalization of the three channels described by means $\mu = (0.485, 0.456, 0.406)$ and standard deviations $\sg = (0.229, 0.224, 0.225)$.
The learning rate was initialized to $0.001$ and decreased to $0.0001$ at epoch $700$, both for AlexNet and MobileNetV2.
The batch size was set to $512$.
The standard deviations regulating the noises of the layers $\ph_{\xx^{\l}}, \l = 1, \dots, L$ were initialized to $\sg_{\om^{\l}}^{f} = \sg_{\om^{\l}}^{b} = \sg_{\bt^{\l}}^{f} = \sg_{\bt^{\l}}^{b} = \sg_{\l} = \sqrt{3}/6$.
On AlexNet, we annealed the forward noises' standard deviations using the \textit{delayed} linear decay
\begin{equation*}
  \sg_{\l} = 1 - \frac{\min(\max(0, t-50(\l)), 50}{50} \,,
\end{equation*}
where $t$ represents the training epoch.
The idea was that $\sg_{\l}$ should have decayed from $\sqrt{3}/6$ at epoch $50\l$ to zero at epoch $50(\l+1)$.
On MobileNetV2, we annealed the forward noises' standard deviations using the linear decay
\begin{equation*}
  \sg_{\l} = 1 - \frac{\min(\max(0, t-50\l), 50}{50} \,.
\end{equation*}
The idea was that $\sg_{\l}$ should decay from $\sqrt{3}/6$ at epoch $50(\l-1)$ to zero at epoch $50\l$.
The standard deviations regulating the backward noises were kept constant at $\sqrt{3}/6$.
\begin{table}[ht]
  \caption{The AlexNet model used in part of the experiments on ImageNet. For each map, $n$ indicates the number of filters, $p$ the (symmetric) padding size, $k$ the (square) kernel spatial side and $s$ the stride in both dimensions.}
  \label{tab:alexnet}
  \centering
  \begin{tabular}{|l|l|l|c|c|c|c|l|} \hline
    Layer & Input Shape & Type & n & p & k & s & Output Shape \\ \hline
    \hline
    \multirow{4}{2em}{$\ph^{1}$} & \multirow{4}{7em}{\threedtnsrshp{227}{227}{3}} & Conv2d      & 64 & 2 & 11 & 4 & \multirow{4}{7em}{\threedtnsrshp{27}{27}{64}} \\ \cline{3-7}
                                 & & MaxPool2d   & - & 0 & 3 & 2 & \\ \cline{3-7}
                                 & & BatchNorm2d & \multicolumn{4}{c|}{-} & \\ \cline{3-7}
                                 & & QuantAct    & \multicolumn{4}{c|}{-} & \\ \hline
    \multirow{4}{2em}{$\ph^{2}$} & \multirow{4}{7em}{\threedtnsrshp{27}{27}{64}} & Conv2d      & 192 & 2 & 5 & 1 & \multirow{4}{7em}{\threedtnsrshp{13}{13}{192}} \\ \cline{3-7}
                                 & & MaxPool2d   & - & 0 & 3 & 2 & \\ \cline{3-7}
                                 & & BatchNorm2d & \multicolumn{4}{c|}{-} & \\ \cline{3-7}
                                 & & QuantAct    & \multicolumn{4}{c|}{-} & \\ \hline
    \multirow{3}{2em}{$\ph^{3}$} & \multirow{3}{7em}{\threedtnsrshp{13}{13}{192}} & Conv2d      & 384 & 1 & 3 & 1 & \multirow{3}{7em}{\threedtnsrshp{13}{13}{384}} \\ \cline{3-7}
                                 & & BatchNorm2d & \multicolumn{4}{c|}{-} & \\ \cline{3-7}
                                 & & QuantAct    & \multicolumn{4}{c|}{-} & \\ \hline
    \multirow{3}{2em}{$\ph^{4}$} & \multirow{3}{7em}{\threedtnsrshp{13}{13}{384}} & Conv2d      & 256 & 1 & 3 & 1 & \multirow{3}{7em}{\threedtnsrshp{13}{13}{256}} \\ \cline{3-7}
                                 & & BatchNorm2d & \multicolumn{4}{c|}{-} & \\ \cline{3-7}
                                 & & QuantAct    & \multicolumn{4}{c|}{-} & \\ \hline
    \multirow{4}{2em}{$\ph^{5}$} & \multirow{4}{7em}{\threedtnsrshp{13}{13}{256}} & Conv2d      & 256 & 1 & 3 & 1 & \multirow{4}{7em}{\threedtnsrshp{6}{6}{256}} \\ \cline{3-7}
                                 & & MaxPool2d   & - & 0 & 3 & 2 & \\ \cline{3-7}
                                 & & BatchNorm2d & \multicolumn{4}{c|}{-} & \\ \cline{3-7}
                                 & & QuantAct    & \multicolumn{4}{c|}{-} & \\ \hline
    \multirow{3}{2em}{$\ph^{6}$} & \multirow{3}{7em}{\onedtnsrshp{9216}} & FC & 4096 & \multicolumn{3}{c|}{-} & \multirow{3}{7em}{\onedtnsrshp{4096}} \\ \cline{3-7}
                                 & & BatchNorm1d & \multicolumn{4}{c|}{-} & \\ \cline{3-7}
                                 & & QuantAct    & \multicolumn{4}{c|}{-} & \\ \hline
    \multirow{3}{2em}{$\ph^{7}$} & \multirow{3}{7em}{\onedtnsrshp{4096}} & FC & 4096 & \multicolumn{3}{c|}{-} & \multirow{3}{7em}{\onedtnsrshp{4096}} \\ \cline{3-7}
                                 & & BatchNorm1d & \multicolumn{4}{c|}{-} & \\ \cline{3-7}
                                 & & QuantAct    & \multicolumn{4}{c|}{-} & \\ \hline
    \multirow{2}{2em}{$\ph^{8}$} & \multirow{2}{7em}{\onedtnsrshp{4096}} & FC & 1000 & \multicolumn{3}{c|}{-} & \multirow{2}{7em}{\onedtnsrshp{1000}} \\ \cline{3-7}
                                 & & BatchNorm1d & \multicolumn{4}{c|}{-} & \\ \hline
  \end{tabular}
\end{table}

\end{document}